\documentclass[a4paper]{article}

\usepackage{aaai}
\usepackage{./macros} 
\usepackage{times}
\usepackage{helvet}
\usepackage{courier}
\usepackage[normalem]{ulem}

\usepackage{amstext,amssymb,amsthm}
\usepackage{amsmath}
\usepackage{url}

\newtheorem{defn}{Definition}
\newtheorem{definition}[defn]{Definition}

\newtheorem{example}{Example}

\newtheorem{proposition}{Proposition}

\newtheorem{remark}{Remark}

\newcommand{\GMax}{\mathrm{GMax}}
\newcommand{\nop}[1]{}

\title{Belief merging within
fragments of propositional logic}
\author{Nadia Creignou \and Odile Papini\\
Aix Marseille Universit\'e, CNRS\\
\And
Stefan R\"ummele \and Stefan Woltran\\
Vienna University of Technology}

\begin{document}
\nocopyright
\maketitle

\begin{abstract}
Recently, belief change within the framework of fragments of propositional logic has gained increasing attention. Previous works focused on belief contraction and belief revision on the Horn fragment. However, the problem of belief merging within
fragments of propositional logic has been neglected so far. This paper presents a general approach to define new merging operators derived from existing ones such that the result of merging remains in the fragment under consideration. Our approach is not limited to the case of Horn fragment but applicable to any fragment of propositional logic characterized by a closure property on the sets of models of its formul\ae. We study the logical properties of the proposed operators in terms of satisfaction of merging postulates, considering in particular  distance-based merging operators for Horn and Krom fragments.%
\end{abstract}

\section{Introduction}
\label{sec:intro}

Belief merging consists in achieving a synthesis between pieces of information
provided by different sources. Although these sources are individually consistent, they may mutually conflict. The aim of merging is to provide a consistent set of information, making maximum use of the information provided by the sources while not favoring any of them. Belief merging is an important issue in many fields of Artificial Intelligence (AI) \cite{IJIS-01} and symbolic approaches to multi-source fusion gave rise to increasing interest within the AI community since the $1990$s \cite{BKM91,LC98,Lin96,Rev93,Rev97}. %
One of today's major approaches is the problem of merging under (integrity) constraints %
in order to generalize both merging (without constraints) and revision (of old information by a new piece of information).
For the latter the constraints then play the role of the new piece of information. 
Postulates characterizing the rational behavior of such merging operators, known as IC postulates, have been proposed by Konieczny and Pino P{\'e}rez~\cite{KP02-JLC} in the same spirit as the seminal AGM \cite{AGM} postulates for revision. %
Concrete merging operators have been proposed according to either semantic (model-based) or syntactic (formula-based) points of view in a classical logic setting \cite{ChaconP12}. %
We focus here on 
the model-based approach of distance-based merging operators \cite{KLM04,KP02-JLC,Rev97}.
These operators are
parametrized by a distance which represents the closeness between interpretations and an aggregation function which captures the merging strategy and takes the origin of beliefs into account. 

Belief change operations within the framework of fragments of classical logic constitute a vivid research branch. In particular, contraction \cite{BMVW11,DW13,ZP12} and revision \cite{DP11,Putte13,ZPZ13} have been thoroughly analyzed in the literature.
The study of belief change within language fragments is motivated by two central observations:
\begin{itemize}
\item
In many applications, the language is restricted a priori. For instance, a rule-based formalization of expert's knowledge is much easier to handle for
standard users. In case users want to revise or merge some sets of rules, they indeed expect that the outcome is still in the easy-to-read format they are used to.
\item 
Many fragments of propositional logic allow for efficient reasoning methods. Suppose an agent  has to make a decision according to a group of experts' beliefs. This should be done efficiently, therefore the expert's beliefs are stored as formul\ae\ known to be in a tractable class. For making a decision, it is desired that the result of the change operation %
yields a set of formul\ae\  in the same fragment. Hence, the agent still can use the dedicated solving method she is equipped with for this
  fragment. 
\end{itemize}

Most of previous work has focused on the Horn fragment except~\cite{CPPW12} that studied revision in any fragment of propositional logic. However, as far as we know, the problem of {\em belief merging} within %
{\em fragments of propositional logic} has been neglected so far. 

The main obstacle hereby is that for a language fragment $\L'$, given $n$ belief bases $K_1, \ldots, K_n \in 2^{\L'}$ and a constraint $\mu \in \L'$, there is no guarantee that the outcome of the merging, 
$\Delta_{\mu}(\{K_1, \ldots, K_n\})$, remains in $\L'$ as well. Let for example, 
$K_1 = \{a\}$,  $K_2 = \{b\}$ and $\mu =\neg a \vee \neg b$ be two sets of  formul\ae\  and a formula expressed in the Horn fragment. Merging with typical distance-based operator proposed in \cite{KP02-JLC} does not remain in the Horn language fragment since 
the result of merging
is equivalent to 
$( a \vee b) \wedge (\neg a \vee \neg b)$, which is not equivalent to any Horn formula (see \cite{Schaefer78}).

We propose the concept of 
\emph{refinement} to overcome these problems.
Refinements have been proposed for revision in \cite{CPPW12} %
and capture the intuition of adapting a given operator (defined for full classical logic) in order to become applicable within a fragment. The basic properties of a refinement aim to (i)  guarantee the result of the change operation to be in the same fragment as the belief change scenario given and (ii)  keep the behavior of the original operator unchanged in case it delivers a result which already fits in the fragment.

Refinements are interesting from different points of view. Several fragments can be treated in a uniform way and a general characterization of refinements is provided for any fragment. Defining and studying refinements of merging operators is not a straightforward extension of the revision case. It is   more complex due to the nature of the merging operators. Even if the constraints play the role of the new piece of information in revision, model-based merging deals with multi-sets of models. Moreover applying this approach to different distance-based merging operators, each parameterized by a  distance and an aggregation function, reveals that all the different parameters matter, thus showing a rich variety of behaviors for refined merging operators.

The main contributions of this paper are the following:
\begin{itemize}
\item We propose to adapt known 
belief merging operators to make them applicable in 
fragments of propositional logic.
We provide natural criteria, which  refined operators should satisfy. We characterize  refined operators in a constructive way.
\item This characterization 
 allows us to study their properties in terms of
the IC postulates \cite{KP02-JLC}. On one hand we prove that the basic postulates (IC0--IC3) are preserved for any refinement for any fragment.
On the other hand we show that the situation is more complex for the remaining postulates. We provide detailed results for the 
Horn and the Krom fragment in terms of 
two kinds of distance-based merging operators  and
three approaches for refinements. 
\end{itemize}

\section{Preliminaries}
\label{sec:prel}

\paragraph{Propositional Logic.}
We consider $\L$ as the language of propositional logic
over some fixed alphabet $\U$ of propositional atoms. 
A literal is an atom or its negation.
A clause is a disjunction of literals. 
A clause is called \emph{Horn} if at most one of its literals is positive; and
\emph{Krom} if it consists of at most two literals.
We identify the following subsets of~$\L$:
$\Lhorn$ is the set of all formul\ae\  in $\L$ being conjunctions of Horn
clauses, and $\Lkrom$ is the set of all formul\ae\  in $\L$ being conjunctions of Krom
clauses.
In what follows we sometimes just talk about arbitrary fragments $\L'\subseteq
\L$.
Hereby, we tacitly assume that any such fragment $\L'\subseteq \L$ contains at
least the formula $\top$.

An interpretation
is represented either by
a set $\omega\subseteq \U$ of atoms (corresponding to the variables set to true) 
or by its
corresponding characteristic bit-vector of length $|\U|$. 
For instance if we consider $\U=\{x_1,\ldots,x_6\}$,
the interpretation $x_1=x_3=x_6=1$ and $x_2=x_4=x_5=0$ will be represented
either by $\{x_1, x_3, x_6\}$ or by $(1,0,1,0,0,1)$.
As usual, if an interpretation $\omega$ satisfies a formula $\phi$, 
we call $\omega$ a model of $\phi$. By $\mod(\phi)$ we denote the set 
of all models (over $\U$) of $\phi$. 
Moreover, $\psi \models \phi$ if $\mod(\psi)\subseteq\mod(\phi)$ 
and $\psi\equiv \phi$ ($\phi$ and $\psi$ are equivalent) if $\mod(\psi)=\mod(\phi)$.

A \emph{base} $K$ is a finite set of propositional formul\ae\ 
$\{\varphi_1,\ldots , \varphi_n\}$.
We shall often identify $K$ via $\bigwedge K$, the conjunction
of formul\ae\ of $K$, i.e., $\bigwedge  K=\varphi_1\land\cdots \land \varphi_n$.
Thus, a base $K$ is said to be consistent if $\bigwedge  K$ is consistent,
$\mod(K)$ is a shortcut for $\mod(\bigwedge K)$,
$K\models \phi$ stands for $\bigwedge K\models \phi$, etc.
Given
$\L'\subseteq\L$ we denote by $\calK_{\L'}$ the set of bases restricted to 
formul\ae\  from $\L'$.
For fragments $\L'\subseteq \L$, we also 
use $T_{{\mathcal L}'}(K) = \{ \phi  \in \L'
\mid K \models \phi\}$.

A \emph{profile} $E$ is a non-empty finite multiset of consistent bases $E=\{K_1,\ldots,
K_n\}$ and represents a group of $n$ agents having different beliefs. 
Given
$\L'\subseteq \L$, 
we denote by $\calE_{\L'}$ the set of profiles restricted to the use of
formul\ae\ from $\L'$.
We denote $\bigwedge K_1\land \ldots\land \bigwedge K_n$ by $\bigwedge E$.
The profile is said to be
consistent if $\bigwedge  E$ is consistent.
 By abuse of notation we write $K\sqcup E$ to denote the multi-set union
$\{K\}\sqcup E$. The multi-set consisting of   the sets of models of the bases in a profile is denoted
$\mmod(E)=\{\mod(K_1),\ldots, \mod(K_n)\}$.
Two profiles $E_1$ and $E_2$ are equivalent, denoted by $E_1\equiv E_2$ if $\mmod(E_1)=\mmod(E_2)$.
Finally, for a set of interpretations $\M$ and a profile $E$ we define 
$\cardintersection(\M,E) = | \{ i : \M \cap\mod(K_i) \not= \emptyset \}|$.

\paragraph{Characterizable Fragments of Propositional Logic.}

Let $\B$ denote the set of all Boolean functions $\op\colon\{0,1\}^k\rightarrow
\{0,1\}$  that have
the following two properties\footnote{these properties are also known as anonimity and unanimity.}:
\begin{itemize}
\item 
\textit{symmetry}, \ie, for all permutations
$\sigma$, $\op(x_1,\ldots, x_k)=\op(x_{\sigma(1)},\ldots, x_{\sigma(k)})$ and
\item 
$0$- and $1$-\textit{reproduction}, \ie,  for all $x\in\{0,1\}$,
$\op(x,\ldots, x)=x$. 
\end{itemize}

Examples are the binary AND function denoted by $\land$ or the ternary MAJORITY function, $\maj_3(x,y,z)=1$ if at least two of the
variables $x,y$, and $z$ are set to 1.
We extend Boolean functions to interpretations by applying
coordinate-wise the original function (recall that we consider interpretations also as bit-vectors). 
So, if $M_1,\ldots, M_k\in\{0,1\}^n$, then
$\op(M_1,\ldots, M_k)$ is defined by
$(\op(M_1[1],\ldots, M_k[1]),\ldots, \op(M_1[n],\ldots,
M_k[n]))$, where $M[i]$ is the $i$-th coordinate of the interpretation $M$.

\begin{definition}
Given a  set ${\cal M}\subseteq 2^\U$ of interpretations and
$\op\in\B$, 
we define $\cl{\op}(\calM)$, the
\emph{closure} of ${\cal M}$ under $\op$, as the smallest set  of interpretations that contains  ${\cal M}$ and that is closed under $\op$, i.e., 
if $M_1,\ldots, M_k\in \cl{\op}(\calM)$, then also $\op(M_1,\ldots,
M_k)\in\cl{\op}(\calM)$.
\end{definition}

\noindent
Let us mention some easy properties of such a  closure:
(i) monotonicity;
(ii) if $\vert \calM\vert =1$, then $\cl{\op}(\calM)=\calM$;
(iii) $\cl{\op}(\emptyset)=\emptyset$.

\begin{definition}
\label{def:fragment}
Let $\op\in\B$.
A set $\L'\subseteq\L$ of propositional formul\ae\ is a  \emph{$\op$-\fragment} (or \emph{characterizable fragment}) if:
\begin{enumerate}
\item for all $\psi \in \L'$, $\mod(\psi) = \cl{\op}(\mod(\psi))$
\item for all $\M\subseteq 2^\U$ with $\M= \cl{\op}(\M)$ there
exists a $\psi \in \L'$  with $\mod(\psi)= \M$
\item if $\phi, \psi \in \L'$ then $\phi \wedge \psi \in \L'$.
\end{enumerate} 
\end{definition}

It is well-known that $\Lhorn$ is an $\land$-fragment and $\Lkrom$ is a $\maj_3$-fragment (see e.g. \cite{Schaefer78}).

\paragraph{Logical Merging Operators.}
Belief merging aims at combining several pieces of information coming from
different sources. Merging operators we consider are functions from the set of
profiles and the set of propositional formul\ae\ to the set of bases, i.e.,
$\Delta\colon \calE_\L\times\L\rightarrow \calK_\L$. For $E\in\calE_\L$ and
$\mu\in\L$ we will write $\Delta_\mu(E)$ instead of $\Delta(E,\mu)$; the
formula $\mu$ is referred to as the \emph{integrity constraint} (IC) and
restricts the result of the merging.

As for
belief revision some logical properties that one could expect from any
reasonable merging operator have been stated.
See \cite{KP02-JLC} for a detailed discussion.
Intuitively $\Delta_\mu(E)$ is
the ``closest'' belief base to the profile $E$ satisfying the integrity constraint
$\mu$. 
This is what the following postulates try to capture. 

{\setlength{\tabcolsep}{4pt}
\begin{tabular}{ll}
$(\ic 0)$  &  $\Delta_\mu(E)\models \mu$ \\
$(\ic 1)$  &  If $\mu$ is consistent, then $\Delta_\mu(E)$ is consistent \\
$(\ic 2)$  &  If $\bigwedge E$ is consistent with $\mu$,
\\ \ ~  & 
then
$\Delta_\mu(E)=\bigwedge E\land\mu$ \\
$(\ic 3)$  & If $E_1\equiv E_2$ and $\mu_1\equiv \mu_2$, 
\\ \ ~  & 
then
$\Delta_{\mu_1}(E_1)\equiv  \Delta_{\mu_2}(E_2)$\\
$(\ic 4)$  &  If $K_1\models \mu$ and   $K_2\models \mu$, then\\  
\ ~  & 
$\Delta_\mu(\{K_1, K_2\})\land K_1$ is consistent if and only if\\
\ ~  &  
$\Delta_\mu(\{K_1, K_2\})\land K_2$ is consistent\\
$(\ic 5)$  &  $\Delta_\mu(E_1)\land \Delta_\mu(E_2)\models \Delta_\mu(E_1\sqcup
E_2)$ \\
$(\ic 6)$  &  If $\Delta_\mu(E_1)\land \Delta_\mu(E_2)$ is consistent, \\ 
\ ~  & then
$\Delta_\mu(E_1\sqcup E_2)\models \Delta_\mu(E_1)\land \Delta_\mu(E_2)$ \\
$(\ic 7)$  &  $\Delta_{\mu_1}(E)\land \mu_2\models \Delta_{\mu_1\land \mu_2}(E)$
\\
$(\ic 8)$  &  If $\Delta_{\mu_1}(E)\land \mu_2$ is consistent, \\ 
\ ~  & then 
$\Delta_{\mu_1\land \mu_2}(E)\models \Delta_{\mu_1}(E)$
\end{tabular}}\\

Similarly to belief revision, a representation theorem \cite{KP02-JLC} shows that a merging operator corresponds to a family of total preorders over interpretations. More formally, for $E\in\calE_\L$, $\mu\in\L$ and $\leq_E$ a total preorder over interpretations, a model-based operator is defined by $\mod(\Delta_{\mu}(E)) = min(\mod(\mu), \leq_E)$. The model-based merging operators select interpretations that are the ''closest'' to the original belief bases.
 
Distance-based operators 
where the notion of closeness stems from
the definition of a distance (or a pseudo-distance\footnote{Let $\omega, \omega' \in \mathcal{W}$, a pseudo-distance is such that $d(\omega, \omega')= d(\omega', \omega)$ and $d(\omega, \omega') = 0$ if and only if $\omega = \omega'$.}) between interpretations and from an aggregation function have been proposed in \cite{KP02-JLC,KP11}. An aggregation function $f$ is a function mapping for any positive integer $n$ each $n$-tuple of positive reals into a positive real such that for any $x_1,\ldots, x_n,x, y\in R^+$, if $x\le y$, then $f(x_1,\ldots, x,\ldots, x_n)\le f(x_1,\ldots, y,\ldots, x_n)$, $f(x_1,\ldots, x_n)=0$ if and only if $x_1=\ldots = x_n=0$ and $f(x)=x$.
Let $E =\{K_1, \ldots, K_n \} \in \calE_\L$, $\mu \in \L$, $d$ be a distance and $f$ be an aggregation function, we consider the family of $\Delta^{d,f}_\mu$ merging operators defined by $\mod(\Delta^{d,f}_\mu(E)) = min(\mod(\mu), \leq_E)$ where $\leq_E$ is a total preorder over the set $2^{\cal U}$ of interpretations defined as follows:
\begin{itemize}
\item $d(\omega,K_i)= min_{\omega' \models K_i}d(\omega,\omega')$,
\item $d(\omega,E) = f(d(\omega,K_1), \ldots, d(\omega,K_n))$, and
\item $\omega \leq_E \omega'$ if $d(\omega,E) \leq d(\omega',E)$. 
\end{itemize}

\begin{definition}\label{def:counting_distance}
 A \emph{counting distance} between interpretations is a function $d: 2^\U \times 2^\U \rightarrow \mathbb{R}^+$ defined for every pair of interpretations $(\omega,\omega')$ by 
$d(\omega,\omega') = g(\card{(\omega\setminus\omega')\cup(\omega'\setminus\omega)}),$
where  $g: \mathbb{N} \rightarrow \mathbb{R}^+$ is a nondecreasing function such that $g(n)=0$ if and only if $n=0$. 
If $g(n)=g(1)$ for every $n\neq 0$,
we call $d$ a
\emph{drastic distance} and denote it via $d_D$.
If $g(n)=n$ for all $n$, 
we call $d$ the
\emph{Hamming distance} and denote it via $d_H$.
If for every interpretations $w,$ $w'$ and $w''$ we have $d(w,w')\le d(w, w'')+d(w'',w')$, then we say that the distance $d$ satisfies the triangular inequality.
 \end{definition}

Observe that  a counting distance is indeed a pseudo-distance, and both, the Hamming distance and drastic distance satisfy the triangular inequality.

As aggregation functions, we consider here $\Sigma$, the sum aggregation function,
and 
the  aggregation function $\GMax$ defined as follows. Let $E = \{K_1, \ldots, K_n \}\in \calE_\L$ and $\omega$, $\omega'$ be two interpretations. Let 
$(d^{\omega}_1, \ldots, d^{\omega}_n )$, where $d^{\omega}_j = d_H(\omega, K_j)$, be the vector of distances between $\omega$ and the $n$ belief bases in $E$. Let $L^E_{\omega}$ be the vector obtained from $(d^{\omega}_1, \ldots, d^{\omega}_n)$ by ranking it in decreasing order. The aggregation function $\GMax$ is defined by $\GMax(d^{\omega}_1, \ldots, d^{\omega}_n)=L^E_{\omega}$, with   $\GMax(d^{\omega}_1, \ldots, d^{\omega}_n) \leq 
\GMax(d^{\omega'}_1, \ldots, d^{\omega'}_n)$ if $L^E_{\omega} \leq_{lex} L^E_{\omega'}$, where $\leq_{lex}$ denotes the lexicographical ordering. 

In this paper we focus on the $\ddelta{\Sigma}$ and $\ddelta{\GMax}$ operators where $d$ is an arbitrary counting distance. 
These operators are known to satisfy 
the postulates $(\ic 0)$--$(\ic 8)$, as shown in \cite{KLM04} generalizing 
more specific results from 
\cite{KP02-JLC,LM98}. 
Finally, we define certain concepts for merging operators and fragments.

\begin{definition}
A \emph{basic} (merging) operator for $\L'\subseteq \L$ is any 
function $\Delta:\calE_{\L'}\times\L'\rightarrow \calK_{\L'}$ satisfying 
$ \mod(\Delta_\mu(\{\{\top\}\}))=\mod(\mu)$ for each $\mu \in \L'$. 
We say that $\Delta$ \emph{satisfies} an $(\ic)$ postulate $(\ic_i)$
($i\in\{0,\ldots,8\}$)
\emph{in $\L'$} if the respective postulate holds when restricted to formul\ae\ 
from $\L'$.
\end{definition}

\section{Refined Operators}\label{sec:approach}

Let us consider a simple example to illustrate the problem 
of standard operators 
when applied within a fragment of propositional 
logic.

\begin{example}
\label{ex:1} 
Let $\U = \{a,b \}$, $E= \{K_1,K_2\} \in \calE_{\Lhorn}$ and $\mu \in \Lhorn$ such that
$\mod(K_1) = \{\{a\},\{a,b\}\}$, $\mod(K_2) = \{\{b\},\{a,b\}\}$,
and $\mod(\mu) = \{\emptyset,\{a\},\{b\}\}$. Consider the distance-based merging operators, $\dHdelta{\Sigma}$ and $\dHdelta{\GMax}$. The following table gives the distances between the interpretations of $\mu$ and the belief bases, and the result of the aggregation functions $\Sigma$ and $\GMax$.

\begin{table}[h]
\centering
\begin{tabular}{c | c | c | c| c}
 $2^{\U}$ & $K_1$ & $K_2$ & $\Sigma$ &$\GMax$\\
  \hline
  $\emptyset$ & $1$ & $1$ & $2$& $ (1, 1)$\\
  $\{a\}$ & $0$ & $1$ & $1$& $(1, 0)$\\
  $\{b\}$ & $1$ & $0$ & $1$& $ (1 , 0)$\\
  \hline
\end{tabular}
\end{table} 
Hence, we have $\mod(\dHdeltamu{\Sigma}(E))=\mod(\dHdeltamu{\GMax}(E))=\{\{a\},\{b\}\}$. Thus, for instance, we can give $\phi = ( a \vee b) \wedge (\neg a \vee \neg b)$ as a result of the merging for both operators.
However, there is no $\psi\in\Lhorn$ with $\mod(\psi)=\{\{a\},\{b\}\}$ (each  $\psi\in\Lhorn$  satisfies the following closure property in terms of its set of models: for every $I,J\in \mod(\psi)$, also $I\cap J\in \mod(\psi)$)). 
Thus, the result of the operator has to be ``refined'', such that it fits
into the Horn fragment.
On the other hand, it holds that 
$\mu \in \Lkrom$, $E \in \calE_{\Lkrom}$ and also the result $\phi$ is in Krom. This shows that different fragments behave differently on certain instances.
Nonetheless, we  aim for a uniform approach for refining merging operators.
\end{example}

We are interested in the following:  Given a known merging operator
$\Delta$ and a fragment $\L'$ of 
propositional logic, 
how can we adapt $\Delta$ 
to a new merging operator $\revdelta$ such that, for each
$E\in\calE_{\L'}$ and $\mu\in\L'$,
 $\revdelta_{\mu}(E)\in\calK_{\L'}$?
Let us define a few natural desiderata for $\revdelta$ 
inspired by the work on belief revision. See \cite{CPPW12} for a discussion.

\begin{definition}\label{def:ref}
Let $\L'$ be a fragment of classical logic and 
$\Delta$ a merging operator.
We call an operator $\revdelta\colon \calE_{\L'}\times\L'\rightarrow
\calK_{\L'}$
a \emph{$\Delta$-refinement for $\L'$} if it satisfies the following
properties, 
for each $E, E_1, E_2\in\calE_{\L'}$ and
$\mu,\mu_1, \mu_2\in\L'$.
\begin{enumerate}
\item consistency: $\Delta_\mu(E)$ is consistent  if and only if
$\revdelta_\mu(E)$ is consistent
\item equivalence:
if $E_1\equiv E_2$ and $\Delta_{\mu_1}(E_1) \equiv \Delta_{\mu_2}(E_2)$ then $\revdelta_{\mu_1}(E_1) \equiv
\revdelta_{\mu_2}(E_2)$
\item
containment: 
$T_{\L'}(\Delta_\mu(E))\subseteq T_{\L'}(\revdelta_\mu(E))$
\item invariance: If $\Delta_\mu(E)\in\calK_{\langle\L'\rangle}$, then
$T_{\L'}(\revdelta_\mu(E)) \subseteq 
T_{\L'}(\Delta_\mu(E) )$, where $\langle\L'\rangle$ denotes the set of formul\ae\  in $\L$ for which there exists an equivalent formula in $\L'$.
\end{enumerate}
\end{definition}

Next we introduce examples of refinements that fit 
Definition \ref{def:ref}.

\begin{definition}
\label{def:first_refinedop}
Let $\Delta$ be a merging operator 
and
$\beta\in\B$.
We define the \emph{$\clop$-based}
refined operator $\Delta^{\clop}$ as:
\begin{align*}
  \mod(\Delta_\mu^{\clop}(E)) = 
\clop(\M).
\end{align*}
where $\M=\mod(\Delta_\mu(E))$.

We define the \emph{$\lex$-based}
refined operator $\Delta^{\lex}$ as:
\begin{align*}
  \mod(\Delta_\mu^{\lex}(E)) =
    \begin{cases}
      \M & \text{if } \clop(\M)=\M,\\
      \{\lex(\M)\} & \text{otherwise},
    \end{cases}
\end{align*}
where $\lex$ is a function
that selects the minimum from a set of interpretations
with respect to a given and fixed order.

We define the \emph{$\lex/\clop$-based}
refined
operator $\Delta^{\lex/\clop}$ as:
\begin{align*}
  \Delta_\mu^{\lex/\clop}(E) =
    \begin{cases}
\Delta_\mu^{\lex}(E) & \text{if } \cardintersection(\M, E)=0 \\                      
      \Delta_\mu^{\clop}(E) & \text{otherwise}.
    \end{cases}
\end{align*}
\end{definition}

The intuition behind the last refinement is to ensure a certain
form of fairness, i.e.\ if no model is selected from the profile, 
this carries over to the refinement.

\begin{proposition}\label{prop:ref}
For any
merging operator
$\Delta:\calE_\L\times\L\rightarrow\calK_\L$,
$\op\in\B$ and
$\L'\subseteq\L$ a $\op$-\fragment,
the operators
$\Delta^{\clop}$,
$\Delta^{\lex}$
 and $\Delta_\mu^{\lex/\clop}$ are
$\Delta$-refinements for $\L'$.
\end{proposition}
\begin{proof}
Let $\mu \in\L'$, $E\in\calE_{\L'}$ and $\op\in\B$.
We show that each operator yields
a base from $\calK_{\L'}$
and moreover satisfies
consistency, equivalence, containment and
invariance, cf.\ Definition \ref{def:ref}.

$\Delta^{\clop}$:
$\Delta^{\clop}_\mu(E)\in\L'$ since
by assumption $\L'$ is a
$\op$-\fragment\ and thus closed under $\op$.
Consistency holds since %
$\mod(\Delta^{\clop}_\mu(E)) = \clop(\mod(\Delta_\mu(E)))$
and $\clop(\M)=\emptyset$ iff $\M=\emptyset$.
Equivalence holds since $\mod(\Delta_{\mu_1}(E_1)) = \mod(\Delta_{\mu_2}(E_2))$ implies
$\clop(\mod(\Delta_{\mu_1}(E_1))) = \clop(\mod(\Delta_{\mu_2}(E_2)))$. Containment: let $\phi \in T_{\L'}(\Delta_\mu(E))$, i.e.\ $\phi \in \L'$ and $\mod(\Delta_\mu(E)) \subseteq \mod(\phi)$.
By monotonicity of $\clop$,
then $\clop(\mod(\Delta_\mu(E)))\subseteq \clop(\mod(\phi))$.
Since $\phi  \in\L'$ then $\clop(\mod(\Delta_\mu(E)))\subseteq \mod(\phi)$ therefore $\phi \in T_{\L'}(\Delta^{\clop}_\mu(E))$.
Invariance: let $\phi \in T_{\L'}(\Delta^{\clop}_\mu(E))$,
i.e.\ $\phi  \in\L'$ and $\clop(\mod(\Delta_\mu(E)))\subseteq \mod(\phi)$.
By hypothesis $\clop(\mod(\Delta_\mu(E))) \supseteq \mod(\Delta_\mu(E))$, therefore $\phi \in T_{\L'}(\Delta_\mu(E))$.

$\Delta^{\lex}$: if $\mod(\Delta^{\lex}_\mu(E))) = \clop(\mod(\Delta_\mu(E)))$
(i.e.\ $\Delta_\mu(E)\in\calK_{\langle \L'\rangle}$)
then $\Delta^{\lex}$ satisfies all the required properties as shown above; otherwise consistency, equivalence and containment hold since $\mod(\Delta^{\lex}_\mu(E))) = \{\lex (\mod(\Delta_\mu(E)))\}$. Moreover, by definition each fragment contains a formula $\phi$ with $\mod(\phi)=\{\omega\}$ where $\omega$ is an arbitrary interpretation. $\Delta_\mu(E)\in\L'$ thus also holds in this case.

$\Delta^{\lex/\clop}$: satisfies the required properties since $\Delta^{\clop}$ and $\Delta^{\lex}$ satisfy them.
\end{proof}

\begin{example}
\label{ex:3}
Consider the profile $E$, the integrity constraint $\mu$  given in Example \ref{ex:1}, the distance-based merging operator $\dHdelta{\Sigma}$, and let $\op$ be the binary AND function. Let us have the following order over the set of interpretations on $\{a, b \}$: $\emptyset < \{ a\} < \{ b \}  < \{ a, b \}$. The result of merging is $\mod(\dHdeltamu{\Sigma}(E))=\{\{a\},\{b\}\}$. The $\lex$-based $\dHdelta{\Sigma}$-refined operator, denoted by $\Delta^{\lex}$, is such that
$\mod(\Delta_\mu^{\lex}(E))=\{\{a\}\}$. The $\clop$-based $\dHdelta{\Sigma}$-refined operator, denoted by $\Delta^{\clop}$, is such that
$\mod(\Delta_\mu^{\clop}(E))= \{\{a\},\{b\}, \emptyset\}$. 
The same result is achieved by the 
 $\lex/\clop$-based $\dHdelta{\Sigma}$-refined operator since $\cardintersection(\mod(\dHdeltamu{\Sigma}(E)), E)=2$. 
\end{example}

In what follows we show how to 
capture not only a particular refined operator 
but characterize the class of \emph{all} refined operators.

\begin{definition}\label{def:op-mapping}
 Given  $\op\in\B$, we define a \emph{$\op$-\mapping}, $\fop$, as an
application which to every  set of models $\calM$ and every multi-set of sets of models $\calX$ associates 
a set of models  
$\fop(\calM, \calX)$ 
such that:

\begin{enumerate}
 \item $\clop(\fop(\calM,\calX))=\fop(\calM,\calX)$
($\fop(\calM,\calX)$ is closed under
$\op$) \label{property:mapclosure}
\item  $\fop(\calM,\calX)\subseteq 
\clop(\calM)$ \label{property:mapinclusion}
\item if $\calM=\clop(\calM)$, then $\fop(\calM,\calX)=\calM$
\label{property:mapinvariance}
\item If $\calM\ne\emptyset$, then $\fop(\calM,\calX)\ne\emptyset$.
\label{property:mapemptyset}
\end{enumerate}
\end{definition}

The concept of mappings allows us to define
a family of refined operators for fragments of classical logic that captures the examples given before.

\begin{definition}\label{def:opclop}
Let $\Delta:\calE_\L\times\L\rightarrow\calK_\L$ be a merging operator and
$\L'\subseteq \L$ be a $\op$-fragment of classical logic
with $\op\in \B$. 
For a $\op$-mapping $\fop$ we denote with
$\Delta^{\fop}:\calE_{\L'}\times\L'\rightarrow\calK_{\L'}$ the operator for $\L'$ defined
as $\mod(\Delta^{\fop}_\mu(E))=\fop(\mod(\Delta_\mu(E)),\mmod(E))$.
The class $[\Delta,\L']$ contains all 
operators $\Delta^{\fop}$ where $\fop$ is a $\op$-mapping  
and $\op\in\B$ such that 
$\L'$ is a $\op$-fragment.
\end{definition}

The next proposition is central in reflecting that the 
above class captures all refined operators we had in mind, cf.\
Definition~\ref{def:ref}.

\begin{proposition}\label{prop:characterization}
Let $\Delta:\calE_\L\times\L\rightarrow\calK_\L$ be a basic merging operator and
$\L'\subseteq \L$ a characterizable fragment of classical logic. 
Then, 
$[\Delta,\L']$
is the set of all 
$\Delta$-refinements for $\L'$.
\end{proposition}

\begin{proof}
Let $\L'$ be a $\op$-fragment for
some $\op\in\B$.
Let $\Delta^\star\in[\Delta,\L']$.
 We show that $\Delta^\star$ is 
a $\Delta$-refinement for $\L'$. 
Let $\mu\in\L'$ and $E\in\calE_{\L'}$.
Since 
$\Delta^\star\in[\Delta,\L']$ there exists a $\op$-mapping $\fop$, such that 
$\mod(\Delta^\star_\mu(E))=\fop(\mod(\Delta_\mu(E)), \mmod(E))$.
By Property \ref{property:mapclosure}
in Definition~\ref{def:op-mapping} 
$\Delta^\star_\mu(E)$ is indeed
in $\calK_{\L'}$.
Consistency: 
If 
$\mod(\Delta_\mu(E))\neq\emptyset$ then 
$\mod(\Delta^\star_\mu(E))
\neq\emptyset$
by Property \ref{property:mapemptyset} in Definition~\ref{def:op-mapping}.
Otherwise,
by Property \ref{property:mapinclusion} in 
Definition~\ref{def:op-mapping}, we get
$\mod(\Delta^\star_\mu(E))
\subseteq
\clop(\mod(\Delta_\mu(E)))=\clop(\emptyset)=\emptyset$.
Equivalence for $\Delta^\star$ is clear by definition and since $\fop$ is defined on 
sets of models.
Containment:
let $\phi\in T_{\L'}(\Delta_\mu(E))$, i.e., $\phi\in\L'$ and 
$\mod(\Delta_\mu(E))\subseteq \mod(\phi)$. We have
$\clop(\mod(\Delta_\mu(E)))\subseteq \clop(\mod(\phi))$ 
by monotonicity  of $\clop$.
By Property \ref{property:mapinclusion} of 
Definition~\ref{def:op-mapping},
$\mod(\Delta^\star_\mu(E)) \subseteq \clop(\mod(\Delta_\mu(E)))$.
Since   $\phi\in\L'$ we have
$\clop(\mod(\phi))=\mod(\phi)$.
Thus, $\mod(\Delta^\star_\mu(E))\subseteq\mod(\phi)$, i.e., $\phi\in T_{\L'}(\Delta^\star_\mu(E))$.
Invariance:
In case
$\Delta_\mu(E)\in\calK_{\langle\L'\rangle}$, we
have
$\clop(\mod(\Delta_\mu(E)))=
 \mod(\Delta_\mu(E))$
since
$\L'$ is a $\op$-fragment.
By 
Property \ref{property:mapinvariance}
in Definition~\ref{def:op-mapping},
we have $\mod(\Delta^\star_\mu(E)) =
\fop(\mod(\Delta_\mu(E)), \mmod(E))=
\mod(\Delta_\mu(E))$.
Thus  $T_{\L'}(\Delta^\star_\mu(E)) \subseteq  T_{\L'}(\Delta_\mu(E))$ as required.

Let $\Delta^\star$ be
a $\Delta$-refinement for $\L'$.  
We show  that
$\Delta^\star\in[\Delta,\L']$.
Let $f$ be defined as follows   for any  set $\M$ of interpretations and $\calX$ a multi-set of sets of interpretations: 
$f(\emptyset, \calX)=\emptyset$. For $\M\ne\emptyset$, if $\clop(\M)=\M$ then $f(\M,\calX)=\calM$, otherwise if there exists
a pair $(E,\mu)\in(\calE_{\L'}, \L')$ such that $\mmod(E)=\calX$ and $\mod(\Delta_\mu(E))=\calM$, then
 we define
$f(\M,\calX)=\mod(\Delta^\star_\mu(E))$.  If there is no such $(E,\mu)$ then we
arbitrarily define $f(\M,\calX)$ as the set consisting of a single model, say the minimal model of $\M$ in the  lexicographic
order. Note that since $\Delta^\star$ is a $\Delta$-refinement for $\L'$, it satisfies
the property of equivalence,
thus the actual choice of the pair $(E,\mu)$ is not 
relevant, and hence $f$ is well-defined.
Thus the   refined operator $\Delta^\star$ behaves like the operator $\Delta^f$.

We show that such a mapping $f$
is a $\op$-mapping.
We  show that the four properties in Definition~\ref{def:op-mapping}
hold for $f$. Property \ref{property:mapclosure} is ensured since for every pair $(\M,\calX)$, $f(\M,\calX)$ is closed under $\op$. Indeed, either $f(\M,\calX)=\M$ if $\M$ is closed under $\op$, or $f(\M,\calX)=\mod(\Delta^\star_\mu(E))$ and since $\Delta^\star_\mu(E)\in\calK_{\L'}$ its set of models is closed under $\op$, or $f(\M,\calX)$ consists of a single interpretation, and thus is also closed under $\op$.
Let us show Property \ref{property:mapinclusion}, i.e.,  $f(\M,\calX)
\subseteq \clop(\M)$ for any pair $(\M,\calX)$.
It is obvious when   $\M=\emptyset$ (then
$f(\M,\calX)=\emptyset$), as well as when $f(\M,\calX)$ is a singleton and when $\M$ is closed and thus $f(\M,\calX)=\M$. Otherwise
$f(\M,\calX)=\mod(\Delta^\star_\mu(E))$ and since $\Delta^\star$ satisfies containment
$\mod(\Delta^\star_\mu(E))\subseteq\clop(\mod(\Delta_\mu(E))$. Therefore  in
any case we have 
$f(\M,\calX)\subseteq \clop(\M)$.  Property \ref{property:mapinvariance}
 follows trivially from the definition of $f(\M,\calX)$ when $\M$ is closed under $\op$.
Property~\ref{property:mapemptyset} is ensured by consistency of $\Delta^\star$.
\end{proof}

Note that the $\op$-mapping which is used in the characterization of refined merging operators differs from the one used in the context of revision (see \cite{CPPW12}).
Indeed, our mapping has  two arguments (and not only one as in the case of revision). %
The additional multi-set of sets of models representing the profile  is required to capture approaches like
the $\lex/\clop$-based refined operator, which are profile dependent. 

\section{IC Postulates}

\begin{table*}
\centering
\begin{tabular}{c|c|c|c|c|c}
& $(\Delta^{d_H,\Sigma})^{\clop}$ & $(\Delta^{d_H,\GMax})^{\clop}$ & $(\Delta^{d_D,x})^{\clop}$ & $(\Delta^{d,x})^\lex$  & $(\Delta^{d,x})^{\lex/\clop}$ \\ 
\hline 
$\ic 4$ & + & - & + & - & + \\ 
\hline 
$\ic 5$, $\ic 7$ & - & - & - & + & - \\ 
\hline 
$\ic 6$, $\ic 8$ & - & - & - & - & -
\end{tabular}
\caption{Overview of results for 
$(\ic 4)$--$(\ic 8)$ for refinements in 
the Horn and Krom fragment ($x\in\{\Sigma,\GMax\}$, $d\in\{d_H,d_D\}$). 
}\label{tab:overview}
\end{table*}

The aim of this section is to study whether refinements of merging operators preserve the IC postulates.
We first show that in case the initial operator satisfies the most basic postulates ($(\ic 0)$--$(\ic 3)$), then so does any of its refinements.
It turns out that this result can not be extended to the remaining postulates. 
For $(\ic 4)$ we characterize a subclass of refinements for which this postulate is preserved.
For the four remaining postulates we  study two representative kinds of distance-based merging operators. We show that
postulates $(\ic 5)$ and $(\ic 7)$ are violated for all of our proposed examples of refined operators with the exception of the $\lex$-based refinement.
For $(\ic 6)$ and $(\ic 8)$ the situation is even worse in the sense that no refinement of our proposed examples of merging operators can satisfy them neither for $\Lhorn$ nor for $\Lkrom$.
Table~\ref{tab:overview} gives
an overview of the results of this section.
However, note that some of the forthcoming results are more general
and hold for arbitrary fragments and/or operators.

\begin{proposition}\label{prop:preserving_postulates}
Let $\Delta$ be a merging operator satisfying postulates $(\ic 0)$--$(\ic 3)$,
and $\L'\subseteq \L$ a characterizable fragment.
Then each $\Delta$-refinement for $\L'$ satisfies $(\ic 0)$--$(\ic 3)$ 
in $\L'$ as well.
\end{proposition}
\begin{proof}
Since $\L'$ is characterizable there exists a $\beta\in\B$, such 
that $\L'$ is a $\op$-fragment. Let $\Delta^\star$ be a $\Delta$-refinement for $\L'$. According to Proposition 
\ref{prop:characterization} we can assume that  $\Delta^\star\in [\Delta,\L']$ is an
operator of form $\Delta^{\fop}$ where $\fop$ is a suitable $\op$-mapping.
In what follows, note that we can restrict ourselves to 
$E  \in \calE_{\L'}$ and to $\mu\in\L'$ since we have to show that $\Delta^{\fop}$ satisfies $(\ic 0)$--$(\ic 3)$
in $\L'$.

$(\ic 0)$: 
Since $\Delta$  satisfies $(\ic 0)$,
$\mod(\Delta_\mu(E) ) \subseteq \mod(\mu)$. Thus, $\clop(\mod(\Delta_\mu(E) )) 
\subseteq \clop(\mod(\mu))$ by monotonicity of the closure. Hence,
$\clop(\mod(\Delta_\mu(E))) 
\subseteq  \mod(\mu)$, since $\mu\in \calL'$ and $\calL'$ is a $\op$-fragment.
According to  Property \ref{property:mapinclusion} in Definition
\ref{def:op-mapping} we have $\fop(\mod(\Delta_\mu(E)),\mmod(E)) \subseteq
\clop(\mod(\Delta_\mu(E))) $, and therefore  by definition of $\Delta^\star_\mu$,
$\mod(\Delta^\star_\mu(E))\subseteq \mod(\mu)$,  which proves that 
$\Delta^\star_\mu(E)\models \mu$.

$(\ic 1)$: Suppose $\mu$ satisfiable. Since $\Delta$ satisfies $(\ic 1)$, $\Delta_\mu(E)$ is satisfiable. Since $\Delta^{\fop}$ is a $\Delta$-refinement 
(Proposition~\ref{prop:characterization}), 
$\Delta_\mu^{\fop}(E)$ is also satisfiable by the property of consistency
(see Definition~\ref{def:ref}).

$(\ic 2)$: Suppose $\bigwedge E$ is consistent with $\mu$. Since $\Delta$ satisfies $(\ic 2)$, $\Delta_\mu(E)=\bigwedge E\land\mu$. We have $\mod(\Delta^\star_\mu(E)) = \fop(\mod(\Delta_\mu(E)),\mmod(E)) = \fop(\mod(\bigwedge E\land\mu),\mmod(E))$. 
Since $\bigwedge E\land\mu \in \L'$ (observe that it is here necessary that the profiles are in the fragment) by Property \ref{property:mapinvariance} of Definition \ref{def:op-mapping} we have $\mod(\Delta^\star_\mu(E)) = \bigwedge E\land\mu$.

$(\ic 3)$: Let $E_1, E_2 \in \calE_{\L'}$ and $\mu_1, \mu_2 \in \L'$ with $E_1 \equiv  E_2$ and $\mu_1 \equiv \mu_2$.  Since $\Delta$ satisfies $(\ic 3)$, $\Delta_{\mu_1}(E_1) \equiv \Delta_{\mu_2}(E_2)$. By the property of equivalence in Definition \ref{def:ref} we have $\Delta^\star_{\mu_1}(E_1) \equiv \Delta^\star_{\mu_2}(E_2)$.
\end{proof}

A natural question is whether   refined operators
for characterizable fragments in their full generality preserve other postulates, and if not whether one can nevertheless find some refined operators that satisfy some of the remaining postulates.

First we show that one can not expect to extend Proposition \ref{prop:preserving_postulates} to $(\ic 4)$. Indeed,  in the two following propositions we exhibit  merging operators which satisfy all postulates, whereas some of  their refinements violate $(\ic 4)$ in  some fragments.

\begin{proposition}\label{prop:IC4_violated_min}
Let  $\Delta$ be  a merging operator with $\Delta \in \{\ddelta{\Sigma},\ddelta{\GMax}\}$, where $d$ is an arbitrary counting distance.
Then the $\lex$-based refined operator $\Delta^{\lex}$ violates postulate  $(\ic 4)$  in $\Lhorn$ and $\Lkrom$. In case $d$ is a drastic distance, $\Delta^{\lex}$ violates postulate  $(\ic 4)$ 
in every characterizable fragment $\L'\subset \L$.
\end{proposition}

\begin{proof}
First consider $d$ is a  drastic distance. 
We show that $\Delta^{\lex}$ violates postulate  $(\ic 4)$
in every characterizable fragment $\L'\subset \L$.
Since $\L'$ is a characterizable fragment there exists $\beta\in\B$ such that $\L'$  is a $\beta$-fragment. Consider a set of models $\M$ that is not closed under $\beta$ and that is cardinality-minimum with this property. Such a set exists since $\L'$ is a proper subset of $\L$.
Observe that necessarily $\card{\M}>1$. Let $m\in\M$, consider the knowledge bases $K_1$ and $K_2$ such that $\mod(K_1)=\{m\}$ and $\mod(K_2)=\M\setminus\{m\}$.
By the choice of $\M$ both $K_1$ and $K_2$ are in $\calK_{\L'}$, whereas $K_1\cup K_2$ is not. Let $\mu=\top$. Since the merging operator uses a drastic distance it is easy to see that $\Delta_\mu(\{K_1, K_2\})= \mod(K_1)\cup \mod(K_2)$. Therefore, $\mod(\Delta_\mu^{\lex}(\{K_1, K_2\}))= \lex(\mod(K_1)\cup \mod(K_2))$, and this single element is either a model of $K_1$ or a model of $K_2$ (but not of both since they do not share any model). This shows that $\Delta^\lex$ violates $(\ic 4)$.

Otherwise, $d$ is defined such that there exists an $x>0$, such that
$g(x)<g(x+1)$.  
We first show that  then
$\Delta^{\lex}$ violates postulate  $(\ic 4)$  in $\Lhorn$. 
Let $A$ be a set of atoms such that $|A|=x-1$ and 
$A\cap\{a,b\}=\emptyset$. 
Moreover,
consider $E=\{K_1, K_2\}$ with $\mod(K_1)=\{\emptyset,\{a\},\{b\}\}$, $\mod(K_2)=\{A\cup \{a,b\}\}$, and let $\mu$ such that 
$\mod(\mu)=\{\emptyset, \{a\}, \{b\}, A\cup\{a,b\}\}$.
Since $g(x)<g(x+1)$, we have $\M=\mod(\Delta_{\mu}(E))=\{\{a\},\{b\},A\cup\{a,b\}\}$, which is not closed under intersection. Hence,  $\mod(\Delta^{\lex}_{\mu}(E))$ contains exactly one of the three models depending on the ordering. Therefore,
$\cardintersection(\mod(\Delta^{\lex}_{\mu}(E)), E)=1$, and thus violating postulate $(\ic 4)$.

For $\Lkrom$, 
let $x>0$ be the smallest index such that 
$g(x)<g(x+1)$ in the definition of distance $d$. 
Note that for any $y$ with $0<y<x$, $g(y)=g(x)$ thus holds.
Let $A,A'$ be two disjoint set of atoms 
with cardinality $x-1$ and $A\cap\{a,b,c,d\}=A'\cap\{a,b,c,d\}=\emptyset$. 
Let us consider $E=\{K_1, K_2\}$ with 
$\mod(K_1)=\{\emptyset,\{a\},\{b\},\{c\},\{d\},\{a,b\},\{c,d\}\}$ (in case $x>1$)
resp.\
$\mod(K_1)=\{\emptyset,\{a\},\{b\},\{c\},\{d\}\}$ (in case $x=1$),
$\mod(K_2)=\{A\cup\{a,b\},A'\cup\{c,d\}\}$, and $\mu$ such that 
$\mod(\mu)=\{\emptyset, \{a\}, \{b\}, \{c\}, \{d\}, \{a,b\}, \{c,d\},A\cup\{a,b\},A'\cup\{c,d\}\}$.
The following table represents the case $x>1$.

\begin{tabular}{l | l | l | l }
    & $K_1$ & $K_2$ &  $E$  \\
    \hline
    $\emptyset$ & $0$ & $g(x+1)$  & $(g(x+1),0)$\\
    $\{a\}$     & $0$ & $g(x)$  & $(g(x),0)$\\
    $\{b\}$      & $0$ & $g(x)$  & $(g(x),0)$\\
    $\{c\}$      & $0$ & $g(x)$  & $(g(x),0)$\\
    $\{d\}$      & $0$ & $g(x)$  & $(g(x),0)$\\
	\hline
    $\{a,b\}$ & $0$ & $g(x-1)$ &  $(g(x-1),0)$\\
    $\{c,d\}$ & $0$ & $g(x-1)$ &  $(g(x-1),0)$\\
	\hline
    $A\cup\{a,b\}$ & $g(x-1)$ & $0$ &  $(g(x-1),0)$\\
    $A'\cup\{c,d\}$ & $g(x-1)$ & $0$ &  $(g(x-1),0)$\\
    \hline
  \end{tabular}

For the case $x>1$,
observe $g(x-1)=g(x)<g(x+1)$, and we have 
$\M=\mod(\Delta_{\mu}(E))=\{\{a\},\{b\},\{c\},\{d\},\{a,b\},\{c,d\},A\cup\{a,b\},A'\cup\{c,d\}\}$. 
For the case $x=1$, note that $A$ and $A'$ are empty, thus the two last rows of the table coincide with the two rows before.
Recall that $K_1$ is defined differently for this case.
Hence, the distances of $\{a,b\}$ and $\{c,d\}$ to $K_1$ are $g(x)=g(1)$.
Thus, we have $\M=\mod(\Delta_{\mu}(E))=\{\{a\},\{b\},\{c\},\{d\},\{a,b\},\{c,d\}\}$.
Neither of the $\M$ 
is closed under ternary majority. Hence,  $\mod(\Delta^{\lex}_{\mu}(E))$ contains exactly one of the six resp.\ eight models depending on the ordering. 
Therefore,
$\cardintersection(\mod(\Delta^{\lex}_{\mu}(E)), E)=1$, thus violating postulate $(\ic 4)$.
\end{proof}

\begin{proposition}\label{prop:ic4_violated_dGmax_cl}
  Let  $\Delta=\ddelta{\GMax}$ be  a merging operator %
where $d$ is an arbitrary 
non-drastic counting distance.
  Then the closure-based refined operator $\Delta^{\clop}$ violates %
  $(\ic 4)$  in $\Lhorn$ and $\Lkrom$.
\end{proposition}

\begin{proof}
Since $d$ is not drastic, there exists an $x>0$ such that 
$g(x)<g(x+1)$. In what follows, we select the smallest such $x$.
We start with the case $\Lhorn$. 
Let $A$ be a set of atoms of cardinality $x-1$ not containing $a,b$.
Let us consider $E=\{K_1, K_2\}$ with $\mod(K_1)=\{\emptyset\}$ and  $\mod(K_2)=\{A\cup\{a,b\}\}$, and $\mu$ such that 
$\mod(\mu)=\{\emptyset, \{a\}, \{b\}, A\cup\{a,b\}\}$.

\begin{tabular}{l | l | l | l }
    & $K_1$ & $K_2$ &  $E$  \\
    \hline
    $\emptyset$ & $0$ & $g(x+1)$  & $(g(x+1),0)$\\
    $\{a\}$     & $g(1)$ & $g(x)$  & $(g(x),g(1))$\\
    $\{b\}$      & $g(1)$ & $g(x)$  & $(g(x),g(1))$\\
    $A\cup\{a,b\}$ & $g(x+1)$ & $0$ &  $(g(x+1),0)$\\
    \hline
  \end{tabular}

Since $g(x)<g(x+1)$, we have $\M=\mod(\Delta_{\mu}(E))=\{\{a\},\{b\}\}$, which is not closed either under intersection.  Hence,  $\mod(\Delta^{\cl{\land}}_{\mu}(E))=\{\{a\},\{b\}, \emptyset\}$. Therefore,
$\cardintersection(\mod(\Delta^{\cl{\land}}_{\mu}(E)), E)=1$, thus violating $(\ic 4)$.

For the case $\Lkrom$,
 let us consider 
two disjoint sets $A,A'$ of atoms not containing $a,b,c,d$ of cardinality $x-1$,
the profile
$E=\{K_1, K_2\}$ with $\mod(K_1)=\{\emptyset\}$ and  $\mod(K_2)=\{A\cup\{a,b\}, A'\cup\{c,d\}\}$, and constraing $\mu$ such that 
$\mod(\mu)=\{\emptyset, \{a\}, \{b\}, \{c\}, \{d\}, \{a,b\}, \{c,d\}, A\cup\{a,b\},A'\cup\{c,d\}\}$.

\begin{tabular}{l | l | l | l }
    & $K_1$ & $K_2$ &  $E$  \\
    \hline
    $\emptyset$ & $0$ & $g(x+1)$  & $(g(x+1),g(0))$\\
    $\{a\}$     & $g(1)$ & $g(x)$  & $(g(x),g(1))$\\
    $\{b\}$      & $g(1)$ & $g(x)$  & $(g(x),g(1))$\\
    $\{c\}$     & $g(1)$ & $g(x)$  & $(g(x),g(1))$\\
    $\{d\}$     & $g(1)$ & $g(x)$  & $(g(x),g(1))$\\
    $\{a,b\}$ & $g(2)$ & $g(x-1)$ &  $(g(x-1),g(2))$\\
    $\{c,d\}$ & $g(2)$ & $g(x-1)$ &  $(g(x-1),g(2))$\\
    $A\cup\{a,b\}$ & $g(x+1)$ & $g(0)$ &  $(g(x+1),g(0))$\\
    $A'\cup\{c,d\}$ & $g(x+1)$ & $g(0)$ &  $(g(x+1),g(0))$\\
    \hline
  \end{tabular}

In case $x=1$ note that $A$ and $A'$ are empty and $g(2)>g(x)>g(x-1)=g(0)$ 
(thus the last four lines collapse into two lines).
We have $\M=\mod(\Delta_{\mu}(E))=\{\{a\},\{b\},\{c\},\{d\}\}$, which is not closed under ternary majority. Hence,  $\mod(\Delta^{\cl{\maj_3}}_{\mu}(E))=\{\{a\},\{b\},\{c\}, \{d\}, \emptyset\}$. 
In case $x>1$, we have $g(x+1)>g(x)=g(x-1)=g(2)=g(1)$.
Thus, $\M=\mod(\Delta_{\mu}(E))=\{\{a\},\{b\},\{c\},\{d\},\{a,b\},\{c,d\}\}$, which is not closed under ternary majority either and one has to add $\emptyset$.
Therefore, in both cases
$\cardintersection(\mod(\Delta^{\cl{\maj_3}}_{\mu}(E)), E)=1$, thus violating $(\ic 4)$.
\end{proof}

In order to identify a class of refinements which satisfy $(\ic 4)$, 
we now introduce the notion of fairness for $\Delta$-refinements.

\begin{definition}
\label{def:fair-ref}
Let $\L'$ be a fragment of classical logic. 
A $\Delta$-refinement for $\L'$, $\revdelta$, is \emph{fair} if it satisfies the following property for each $E%
\in\calE_{\L'}$, $\mu \in\L'$:
If 
$\cardintersection(\Delta_\mu(E),E) \neq 1$ then
$\cardintersection(\revdelta_\mu(E),E) \neq 1$.
\end{definition}

\begin{proposition}\label{prop:ex_fair_operators}
Let $\L'$ be a characterizable fragment. 
(1) The $\clop$-based refinement of both $\drasticdelta{\Sigma}$ and  $\drasticdelta{\GMax}$ for $\L'$ is fair. 
(2) The $\lex/\clop$-based refinement of any merging operator for $\L'$ is fair.
\end{proposition}

\begin{proof}  Let $\L'$ be a $\op$-fragment. 
Let $E \in \cal E_{\L'}$ such that $E = \{K_1, \ldots K_n\}$, $\mu \in \L'$ and  
let $\Delta$ be 
$\drasticdelta{\Sigma}$ or $\drasticdelta{\GMax}$ for case (1), resp.\ 
let $\Delta$ be
an arbitray merging operator in case of (2).

$\Delta^{\clop}$: If $\cardintersection(\Delta_\mu(E),E)> 1$ then,  $\cardintersection(\clop(\Delta_\mu(E)),E)\ge \cardintersection(\Delta_\mu(E),E)> 1$.
Since the drastic distance is used observe that for any model $m$ of $\mu$ we have $d(m,E)=n-\vert\{i\mid m\in K_i\}\vert$. Thus, if $\cardintersection(\Delta_\mu(E),E)=0$, then 
$\mod(\Delta_\mu(E))\cap\bigcup_i\mod(K_i)=\emptyset$, and thus $\mod(\Delta_\mu(E))=\mod(\mu)$. In this case $\mod(\Delta^{\clop}_\mu(E))=\mod(\Delta_\mu(E))$ and therefore  $\cardintersection(\Delta^{\clop}_\mu(E),E)=0$ as well.

$\Delta^{\lex/\clop}$:
If $\cardintersection(\Delta_\mu(E),E)= 0$ then
$\mod(\Delta_\mu(E))\cap \bigcup_i \mod(K_i)= \emptyset$. By Definition \ref{def:first_refinedop} 
$\Delta^{\lex/\clop}_\mu(E) = \Delta_\mu^{\lex}(E)$, therefore 
$\cardintersection(\Delta^{\lex/\clop}_\mu(E),E)= 0$  as well.
If 
$\cardintersection(\Delta_\mu(E),E)> 1$ then
by Definition \ref{def:first_refinedop}, 
$\mod(\Delta^{\lex/\clop}_\mu) = \mod(\Delta_\mu^{\clop}(E))$, 
thus 
$\cardintersection(\Delta^{\lex/\clop}_\mu(E),E) \geq \cardintersection(\Delta_\mu(E),E) > 1$.
\end{proof}

Fairness turns out to be a sufficient property to preserve
 the postulate $(\ic 4)$ as stated in the following proposition.

\begin{proposition}\label{prop:ic4_preserved_dsigma}
Let $\Delta$ be a merging operator satisfying postulate $(\ic 4)$,
and $\L'\subseteq \L$ a characterizable fragment.
Then every fair $\Delta$-refinement for $\L'$ satisfies $(\ic 4)$ as well.
\end{proposition}
\begin{proof}
 Consider $\Delta$ a merging operator satisfying postulate $(\ic 4)$. Let $\revdelta$ be a fair $\Delta$-refinement for $\L'$. If $\revdelta$ does not satisfy $(\ic 4)$, then there exist $E=\{K_1, K_2\}$ with $K_1, K_2\in\L'$ and  $\mu\in\L'$, with
 $K_1\models \mu$ and $K_2\models \mu$ such that $\mod(\revdelta_\mu(E)) \cap \mod(K_1) \not= \emptyset $ and $\mod(\revdelta_\mu(E)) \cap \mod(K_2) = \emptyset$, i.e., such that $\cardintersection(\revdelta_\mu(E),E)=1$. Since  $\Delta$ satisfies postulate $(\ic 4)$ we have 
 $\cardintersection(\Delta_\mu(E),E)\ne 1$, thus contradicting  the fairness property in Definition \ref{def:fair-ref}.
 \end{proof}

With the above result at hand, we can conclude
that 
the 
$\clop$-based refinement of both $\drasticdelta{\Sigma}$ and  $\drasticdelta{\GMax}$ for $\L'$ 
as well as 
the $\lex/\clop$-based refinement of any merging operator 
satisfies $(\ic 4)$.

\begin{remark}
 Observe that the distance which is used in distance-based operators matters with respect to the preservation of  $(\ic 4)$, as well as for fairness. 
Indeed, while the $\clop$-refinement of  $\drasticdelta{\GMax}$ is fair, and therefore satisfies $(\ic 4)$, the  $\clop$-refinement of 
$\ddelta{\GMax}$   where $d$ is an arbitrary 
non-drastic counting distance violates postulate  $(\ic 4)$  in $\Lhorn$ and $\Lkrom$, and therefore is not fair.
\end{remark}

For all refinements considered so far we know whether $(\ic 4)$ is preserved or not, with one single exception: 
the  $\clop$-refinement of 
$\ddelta{\Sigma}$   where $d$ is an arbitrary 
non-drastic counting distance. In this case we get a partial positive result.

\begin{proposition}\label{prop:closure_preserves_ic4}
  Let  $\Delta$ be  a merging operator with $\Delta =\ddelta{\Sigma}$, where $d$ is an arbitrary 
counting distance that satisfies the 
triangular inequality.
  Then the closure-based refined operator $\Delta^{\clop}$ satisfies postulate  $(\ic 4)$  in any characterizable fragment.
\end{proposition}

\begin{proof} Let $\L'$ be a $\op$-fragment.
 Let $E=\{K_1, K_2\}$ with $K_1, K_2\in\L'$ and  $\mu\in\L'$, with
$K_1\models \mu$ and $K_2\models \mu$. The merging operator $\Delta$ satisfies $(\ic 4)$ therefore
$\Delta_\mu(E)\land K_1$ is consistent if and only if $\Delta_\mu(E)\land K_2$. 

If both $\Delta_\mu(E)\land K_1$ and $\Delta_\mu(E)\land K_2$ are consistent, then so are \textit{a fortiori}
$\Delta^{\clop}_\mu(E)\land K_1$ and $\Delta^{\clop}_\mu(E)\land K_2$. Therefore
a violation of $(\ic 4)$ can only occur when  both $\Delta_\mu(E)\land K_1$ and $\Delta_\mu(E)\land K_2$ are inconsistent.
We prove that this never occurs. Suppose that $\Delta_\mu(E)\land K_1$ is inconsistent, this means that there exists 
$m\not\in K_1$ such that $min(\mod(\mu), \leq_E)=d(m,E)$
and that for all $m_1\in K_1$,
$d(m,E)<d(m_1,E)$, i.e., $d(m,K_1)+d(m,K_2)<d(m_1,K_1)+d(m_1,K_2)$ since $\Sigma$ is the aggregation function.
Choose now $m_1\in K_1$ such that $d(m,K_1)=d(m,m_1)$ and $m_2\in K_2$ such that $d(m,K_2)=d(m,m_2)$.
We have $d(m,K_1)+d(m,K_2)=d(m,m_1)+d(m,m_2)<d(m_1,K_1)+d(m_1,K_2)=d(m_1,K_2)$ since $m_1\in K_1$ and hence $d(m_1,K_1)=0$.
Since $d$ satisfies the  triangular inequality we have $d(m_1,m_2)\le d(m_1,m)+ d(m,m_2)$.
But this contradicts $d(m,m_1)+d(m,m_2) < d(m_1,K_2) \leq d(m_1,m_2)$, thus $\Delta_\mu(E)\land K_1$ can not be inconsistent.
\end{proof}

\begin{remark}
The above proposition together with Proposition \ref{prop:ic4_violated_dGmax_cl} shows that the aggregation function that is used in distance-based operators matters with respect to the preservation of the postulate $(\ic 4)$.
\end{remark}

Interestingly Proposition \ref{prop:closure_preserves_ic4} (recall that the Hamming distance satisfies the triangular
inequality) together with the following proposition  show that fairness, which is a sufficient condition for preserving $(\ic 4)$ is not a necessary one.

\begin{proposition}\label{prop:cl_non_fair}
The $\clop$-refinement of $\dHdelta{\Sigma}$ is not fair in $\Lhorn$ and in $\Lkrom$.
\end{proposition}

\begin{proof}
We give the proof for $\Lhorn$.
One can verify that the same example works for $\Lkrom$ as well.

Let us consider $E=\{K_1,K_2\}$ and $\mu$ in $\Lhorn$ with 
$\mod(K_1)=  \{\{a\},\{a, b\}, \{a,d\}, \{a,f\}\}$, 
$\mod(K_2) =  \{ \{a, b, c, d, e, f, g\}\}$ and 
$\mod(\mu) =\{\{a\},\{a, b,c\}, \{a,d,e\}, \{a,f,g\}\}.$
We have $\mod(\dHdelta{\Sigma}_{\mu}(E))=\{ \{a, b,c\}, \{a,d,e\}, \{a,f,g\}\}$, and
 $\mod(\Delta^{\cl{\land}}_{\mu}(E))=\{ \{a\},\{a, b,c\}, \{a,d,e\}, \{a,f,g\}\}$.
Therefore, $\cardintersection(\mod(\dHdelta{\Sigma}_{\mu}(E)), E)=0$, whereas $\cardintersection(\mod(\Delta^{\cl{\land}}_{\mu}(E)), E)=1$,  
thus proving that fairness is not satisfied.
\end{proof}

It turns out that our refined operators have a similar behavior with respect to postulates $(\ic 5)$ \& $(\ic 7)$ as well as $(\ic 6)$ \& $(\ic 8)$.
Therefore we will deal with the remaining postulates in pairs.
In fact the $\lex$-based refinement satisfies $(\ic 5)$ and $(\ic 7)$, whereas the refined operators $\Delta^{\clop}$ and $\Delta^{\lex/\clop}$ violate these two postulates.

\begin{proposition}\label{prop:ic57_sat_min}
Let $\Delta$ be a merging operator satisfying postulates $(\ic 5)$ and $(\ic 6)$ (resp.\ $(\ic 7)$ and $(\ic 8)$),
and $\L'\subseteq \L$ a characterizable fragment.
Then the refined operator $\Delta^\lex$ for $\L'$ satisfies $(\ic 5)$ (resp.\ $(\ic 7)$) in $\L'$ as well.
\end{proposition}

\begin{proof}
Since $\L'$ is characterizable there exists a $\beta\in\B$, such 
that $\L'$ is a $\op$-fragment.

$(\ic 5)$:
If $\Delta_\mu^\lex(E_1)\wedge \Delta_\mu^\lex(E_2)$ is inconsistent, then $(\ic 5)$ is satisfied.
Assume that $\Delta_\mu^\lex(E_1)\wedge \Delta_\mu^\lex(E_2)$ is consistent.
Then, by definition of $\Delta^\lex$ we know that $\Delta_\mu(E_1)\wedge \Delta_\mu(E_2)$ is consistent as well.
From $(\ic 5)$ and $(\ic 6)$ it follows that $\mod(\Delta_\mu(E_1))\cap \mod(\Delta_\mu(E_2))=\mod(\Delta_\mu(E_1\sqcup E_2))$.
We distinguish two cases.
First assume that both $\mod(\Delta_\mu(E_1))$ and $\mod(\Delta_\mu(E_2))$ are closed under $\beta$.
By Definition~\ref{def:fragment} we know that $\mod(\Delta_\mu(E_1))\cap \mod(\Delta_\mu(E_2))=\mod(\Delta_\mu(E_1\sqcup E_2))$ is closed under $\beta$ as well.
Hence, $(\ic 5)$ is satisfied.
For the second case assume that not both $\mod(\Delta_\mu(E_1))$ and $\mod(\Delta_\mu(E_2))$ are closed under $\beta$.
From the definition of $\Delta^\lex$ it follows that $\mod(\Delta_\mu^\lex(E_1))\cap \mod(\Delta_\mu^\lex(E_2))$ consists of a single interpretation, say $I$ with $I \in \mod(\Delta_\mu(E_1))\cap \mod(\Delta_\mu(E_2))$.
If $\mod(\Delta_\mu(E_1\sqcup E_2))$ is closed under $\beta$ we have $I \in \mod(\Delta_\mu^\lex(E_1\sqcup E_2))$ and $(\ic 5)$ is satisfied.
If $\mod(\Delta_\mu(E_1\sqcup E_2))$ is not closed under $\beta$, then  $\mod(\Delta_\mu^\lex(E_1\sqcup E_2))$ consists of a single interpretation, say $J \in \mod(\Delta_\mu(E_1))\cap \mod(\Delta_\mu(E_2))$.
From $\mod(\Delta_\mu^\lex(E_1))\cap \mod^\lex(\Delta_\mu(E_2)) = \{I\}$ it follows that $\lex(\{I,J\})=I$ and from $\mod(\Delta_\mu^\lex(E_1\sqcup E_2)) = \{J\}$ it follows that $\lex(\{I,J\})=J$.
Hence, $I=J$ and $(\ic 5)$ is satisfied.

$(\ic 7)$:
If $\Delta_{\mu_1}^\lex(E)\wedge \mu_2$ is inconsistent, then $(\ic 7)$ is satisfied.
Assume that $\Delta_{\mu_1}^\lex(E)\wedge \mu_2$ is consistent.
Then, by definition of $\Delta^\lex$ we know that $\Delta_{\mu_1}(E)\wedge \mu_2$ is consistent as well.
From $(\ic 7)$ and $(\ic 8)$ it follows that $\mod(\Delta_{\mu_1}(E))\cap\mod(\mu_2)=\mod(\Delta_{\mu_1\wedge\mu_2}(E))$.
We distinguish two cases.
First assume that $\mod(\Delta_{\mu_1}(E))$ is closed under $\beta$.
By Definition~\ref{def:fragment} we know that $\mod(\Delta_{\mu_1}(E))\cap\mod(\mu_2)=\mod(\Delta_{\mu_1\wedge\mu_2}(E))$ is closed under $\beta$ as well.
Hence, $(\ic 7)$ is satisfied.
For the second case assume that $\mod(\Delta_{\mu_1}(E))$ is not closed under $\beta$.
From the definition of $\Delta^\lex$ it follows that $\mod(\Delta_{\mu_1}^\lex(E))\cap \mod(\mu_2)$ consists of a single interpretation, say $I$ with $I \in \mod(\Delta_{\mu_1}(E))\cap\mod(\mu_2)$.
If $\mod(\Delta_{\mu_1\wedge\mu_2}(E))$ is closed under $\beta$ we have $I \in \mod(\Delta_{\mu_1\wedge\mu_2}^\lex(E))$ and $(\ic 7)$ is satisfied.
If $\mod(\Delta_{\mu_1\wedge\mu_2}(E))$ is not closed under $\beta$, then  $\mod(\Delta_{\mu_1\wedge\mu_2}^\lex(E))$ consists of a single interpretation, say $J \in \mod(\Delta_{\mu_1}(E))\cap\mod(\mu_2)$.
From $\mod(\Delta_{\mu_1}^\lex(E))\cap \mod(\mu_2) = \{I\}$ it follows that $\lex(\{I,J\})=I$ and from $\mod(\Delta_{\mu_1\wedge\mu_2}^\lex(E)) = \{J\}$ it follows that $\lex(\{I,J\})=J$.
Hence, $I=J$ and $(\ic 7)$ is satisfied.
\end{proof}

\begin{proposition}\label{prop:IC567_violated_for_closure}
 Let  $\Delta$ be  a merging operator with $\Delta \in\{\ddelta{\Sigma}, \ddelta{\GMax}\}$, where $d$ is an arbitrary 
counting distance.
  Then the refined operators $\Delta^{\clop}$ and $\Delta^{\lex/\clop}$ violate postulates $(\ic 5)$ and $(\ic 7)$ in $\Lhorn$ and in $\Lkrom$.
\end{proposition}

\begin{proof} We give the proof for $\Delta^{\clop}$ with $\Delta=\Delta^{d,\Sigma}$ where $d$ is %
associated with a function $g$.
The given examples also apply to $\GMax$
and for %
the refinement
$\Delta^{\lex/\clop}$. %

$(\ic 5)$:
Let $\beta\in\{\wedge,\maj_3\}$.
Consider $E_1=\{K_1,K_2,K_3\}$, $E_2=\{K_4\}$ and $\mu$ with 
$\mod(K_1) =  \{\{a\},\{a, b\},\{a, c\} \}$,
$\mod(K_2) =  \{\{b\}, \{a,b\},\{b, c\}\}$,
$\mod(K_3) =  \{\{c\}, \{a,c\},\{b, c\}\}$,
$\mod(K_4) =  \{\emptyset,\{b\}\}$, and
$\mod(\mu) =  \{\emptyset,\{a\}, \{b\},\{c\}\}$.

\begin{tabular}{l | l | l | l | l | l | l}
    & $K_1$ & $K_2$ & $K_3$ & $K_4$ & $E_1$& $E_1\sqcup E_2$\\
    \hline
    $\emptyset$ & $g(1)$ & $g(1)$ & $g(1)$ & 0 & $3g(1)$ & $3g(1)$\\
    $\{a\}$     & $0$ & $g(1)$ & $g(1)$ &  $g(1)$ & $2g(1)$ & $3g(1)$\\
    $\{b\}$     & $g(1)$ & $0$ & $g(1)$ & 0 &  $2g(1)$ & $2g(1)$\\
    $\{c\}$     & $g(1)$ & $g(1)$ & $0$ & $g(1)$ &  $2g(1)$ & $3g(1)$\\
  
    \hline
  \end{tabular}

\noindent
Since $g(1)>0$ by definition of a counting distance, we have 
  $\mod(\Delta^{\clop}_\mu(E_1)) = \{\emptyset,\{a\},\{b\},\{c\}\}$,
  $\mod(\Delta^{\clop}_\mu(E_2)) = \{\emptyset,\{b\}\}$, and
  $\mod(\Delta^{\clop}_\mu(E_1\sqcup E_2)) = \{\{b\}\}$,
  violating $(\ic 5)$.

$(\ic 7)$: %
For $\Lhorn$,
consider %
$E = \{K_1,K_2,K_3\}$ %
with 
  $\mod(K_1) = \{\{a\}\}$, $\mod(K_2) = \{\{b\}\}$, $\mod(K_3) = \{\{a,b\}\}$,
and assume $\mod(\mu_1) = \{\emptyset,\{a\},\{b\}\}$ and $\mod(\mu_2) = \{\emptyset,\{a\}\}$.

  \begin{tabular}{l | l | l | l | l }
    & $K_1$ & $K_2$ & $K_3$ & $E$ \\
    \hline
    $\emptyset$ & $g(1)$ & $g(1)$ & $g(2)$ & $2g(1)+g(2)$\\
    $\{a\}$ & $0$ & $g(2)$ & $g(1)$ & $g(1)+g(2)$\\
    $\{b\}$ & $g(2)$ & $0$ & $g(1)$ & $g(1)+g(2)$\\
    \hline
  \end{tabular}

\noindent
We have $\mod(\Delta_{\mu_1}(E))=\{\{a\},\{b\}\}$,   thus $\mod(\Delta^{Cl_\wedge}_{\mu_1}(E))=\{\emptyset,\{a\},\{b\}\}$. Therefore, $\mod(\Delta^{Cl_\wedge}_{\mu_1}(E)\wedge\mu_2)=\{\emptyset,\{a\}\}$, whereas $\mod(\Delta^{Cl_\wedge}_{\mu_1\land\mu_2}(E))=\{\{a\}\}$,
violating $(\ic 7)$.

For $\Lkrom$ let
$E=\{K_1,K_2,K_3, K_4, K_5\}$,  $\mu_1$ and  $\mu_2$ with 
$\mod(K_1) =  \{\{a\} \}$,
$\mod(K_2) =  \{\{b\} \}$,
$\mod(K_3) = \{\{c\}\}$,
$\mod(K_4) =  \{\{a, b\},\{a,c\} \}$,
$\mod(K_5) =  \{\{a, b\},\{b, c\}\}$,
$\mod(\mu_1) =  \{\emptyset,\{a\}, \{b\},\{c\}\}$, and
$\mod(\mu_2) =  \{\emptyset,\{a\}\}$.

{\setlength{\tabcolsep}{4pt}
\begin{tabular}{l | l | l | l | l | l | l }
    & $K_1$ & $K_2$ & $K_3$ & $K_4$ & $K_5$ &  $E$\\
    \hline
    $\emptyset$ & $g(1)$ & $g(1)$ & $g(1)$ & $g(2)$ & $g(2)$ & $2g(2)+3g(1)$\\
    $\{a\}$     & $0$ & $g(2)$ & $g(2)$ & $g(1)$ & $g(1)$ & $2g(2)+2g(1)$\\
    $\{b\}$     & $g(2)$ & $0$ & $g(2)$ & $g(1)$ & $g(1)$ & $2g(2)+2g(1)$\\
    $\{c\}$     & $g(2)$ & $g(2)$ & $0$ & $g(1)$ & $g(1)$ & $2g(2)+2g(1)$\\
    \hline
  \end{tabular}}
 We have
  $\mod(\Delta^{Cl_{\maj_3}}_{\mu_1}(E)) = \{\emptyset,\{a\},\{b\}, \{c\}\}$, 
thus $\mod(\Delta^{Cl_{\maj_3}}_{\mu_1}(E)\land \mu_2) = \{\emptyset,\{a\}\}$, and
   $\mod(\Delta^{Cl_{\maj_3}}_{\mu_1\wedge\mu_2}(E)) = \{\{a\}\}$.
  This violates postulate $(\ic 7)$.
\end{proof}

Actually in the Horn fragment the negative results of the above proposition can be extended to any fair refinement.

\begin{proposition}\label{prop:IC567_violated_for_fair_inHorn}
 Let  $\Delta$ be  a merging operator with $\Delta \in\{\ddelta{\Sigma}, \ddelta{\GMax}\}$, where $d$ is an arbitrary 
counting distance.
  Then any fair refined operator $\Delta^*$ violates postulates $(\ic 5)$ and $(\ic 7)$ in $\Lhorn$.
\end{proposition}

\begin{proof}
 The same, or simpler examples as in the proof of the previous proposition will work here.
 We give the proof  in the case of $\Delta^{d,\Sigma}$ where $d$ is a counting distance associated with the function $g$.
It is easy to see that the given examples work as well when using the aggregation function $\GMax$. It can  be observed in the following  that any involved set of models is closed under intersection  and hence it can be represented by a  Horn formula.

$(\ic 5)$: 
Let us consider $E_1=\{K_1,K_2\}$, $E_2=\{K_3\}$ and $\mu$ with 
$\mod(K_1)=  \{\{a\},\{a, b\} \}$,$ \mod(K_2) =  \{\{b\}, \{a,b\}\}$,
$\mod(K_3) =  \{\emptyset,\{b\}\}$ and 
$ \mod(\mu)  = \{\emptyset,\{a\}, \{b\}\}.$
Since $g(1)>0$ by definition of a counting distance, we have $\mod(\Delta_\mu(E_1)) = \{\{a\},\{b\}\}$,
and thus
  $\mod(\Delta^{*}_\mu(E_1)) \subseteq \{\emptyset,\{a\},\{b\}\}$.
We can exclude $\mod(\Delta^{*}_\mu(E_1)) = \{\{a\},\{b\}\}$ since it is not closed under $\wedge$.
  By Definition~\ref{def:fair-ref} we can exclude $\mod(\Delta^{*}_\mu(E_1)) = \{\{a\}\}$ and $\mod(\revdelta_{\mu}(E_1)) = \{\{b\}\}$.
Therefore either $\mod(\Delta^{*}_\mu(E_1)) = \{\emptyset\}$ or $\mod(\Delta^{*}_\mu(E_1)) =\{\emptyset,\{a\},\{b\}\}$.
On the one hand, since
  $\mod(\Delta^{*}_\mu(E_2)) = \{\emptyset,\{b\}\}$, in any case $\mod(\Delta^{*}_\mu(E_1)\land \Delta^{*}_\mu(E_2))$ contains $\emptyset$.  
On the other hand
  $\mod(\Delta^{*}_\mu(E_1\sqcup E_2)) = \{\{b\}\}$.
  This violates postulate $(\ic 5)$.

$(\ic 7)$:
  There we have $\mod(\Delta _{\mu_1\wedge\mu_2}(E))=\{\{a\}\}$.
  By properties 3 and 4 of Definition~\ref{def:ref} it holds $\mod(\revdelta_{\mu_1\wedge\mu_2}(E)) = \{\{a\}\}$.
  Since $\mod(\Delta_{\mu_1}(E))=\{\{a\},\{b\}\}$, it follows that $\mod(\revdelta_{\mu_1}(E)) \subseteq \{\emptyset,\{a\},\{b\}\}$.
  We can exclude $\mod(\revdelta_{\mu_1}(E)) = \{\{a\},\{b\}\}$ since it is not closed under $\wedge$.
  By Definition~\ref{def:fair-ref} we can exclude $\mod(\revdelta_{\mu_1}(E)) = \{\{a\}\}$ and $\mod(\revdelta_{\mu_1}(E)) = \{\{b\}\}$.
  Hence, $\emptyset \in \mod(\revdelta_{\mu_1}(E))$.
  Therefore $\emptyset \in \mod(\revdelta_{\mu_1}(E)) \cap \mod(\mu_2)$ but $\emptyset \not\in \mod(\revdelta_{\mu_1\wedge\mu_2}(E))$ which violates $(\ic 7)$.
\end{proof}

We leave it as an open question whether this proposition can be extended to Krom.
For the two remaining postulates, $(\ic 6)$ and $(\ic 8)$, the situation is even worse, since any refinement of the two kinds of distance-based merging operators we considered violates them in $\Lhorn$ and in $\Lkrom$.

\begin{proposition}\label{prop:IC6_IC8_violated}
Let  $\Delta$ be  a merging operator with $\Delta \in\{\ddelta{\Sigma}, \ddelta{\GMax}\}$, where $d$ is an arbitrary 
counting distance.
  Then any refined operator $\revdelta$ violates postulates $(\ic 6)$ and $(\ic 8)$ in $\Lhorn$ and in $\Lkrom$.
\end{proposition}
\begin{proof}
 As an example we give the proof for $(\ic 6)$ in $\Lhorn$ for $\ddelta{\GMax}$.
Since $\Lhorn$ is an $\land$-fragment,
there is an $\land$-mapping $f$  such that $\revdelta=\Delta^f$
	and we have $f(\M,{\cal X})\subseteq \cl{\land}(\M)$ with $\cl{\land}(f(\M,{\cal X}))=f(\M,{\cal X})$.
	Let us consider $E_1=\{K_1,K_2,K_3\}$ and $\mu$ with $\mod(K_1) = \{\{a\},\{a, b\}\}$, $\mod(K_2)=  \{\{b\},\{a, b\}\}$, 
$\mod(K_3) = \{\emptyset,\{a\}, \{b\}\}$ and $\mod(\mu)  =  \{\emptyset,\{a\}, \{b\},\{a,b\}\}.$
\begin{tabular}{l | l | l | l | l}
    & $K_1$ & $K_2$ & $K_3$ &  $E_1$\\
    \hline
    $\emptyset$ & $g(1)$ & $g(1)$ & $0$ & $(g(1), g(1), 0)$\\
    $\{a\}$     & $0$ & $g(1)$ & $0$   & $(g(1), 0, 0)$\\
    $\{b\}$     & $g(1)$ & $0$ & $0$ & $(g(1), 0, 0)$\\
    $\{a,b\}$     & $0$ & $0$ & $g(1)$ & $(g(1), 0, 0)$\\
    \hline
  \end{tabular}

\noindent
We have $\M=\mod(\Delta_{\mu}(E_1))=\{\{a\},\{b\},\{a,b\}\}$. 
Let us consider the possibilities for $\mod(\revdelta_{\mu}(E_1))=f(\M,\mmod(E_1))$.
If $\emptyset\in f(\M,\mmod(E_1))$, then let
 $E_2=\{K_4\}$ with $K_4$ in $\Lhorn$ be such that $\mod(K_4)=\{\emptyset\}$. 
Thus, $\mod(\revdelta_{\mu}(E_2)) = \{\emptyset\}$ and $\mod(\revdelta_{\mu}(E_1) \wedge \revdelta_{\mu}(E_2)) = \{\emptyset\}$.
Moreover, $\mod(\Delta_{\mu}(E_1\sqcup E_2))= \{\emptyset,\{a\},\{b\}\}$ or $\{\emptyset,\{a\},\{b\},\{a,b\}\}$ depending on whether $g(1) < g(2)$ or $g(1) = g(2)$. Since both sets are closed under intersection, we have $\mod(\revdelta_{\mu}(E_1\sqcup E_2))=\mod(\Delta_{\mu}(E_1\sqcup E_2))$.
Thus $\mod(\revdelta_{\mu}(E_1\sqcup E_2)) \not\subseteq \{\emptyset\}$ and $(\ic 6)$ does not hold.

Otherwise, $f(\M,\mmod(E_1))\subseteq \{\{a\},\{b\},\{a,b\}\}$. 
By symmetry assume w.l.o.g. that $f(\M,\mmod(E_1))\subseteq \{\{a, b\},\{a\}\}$ (note that $\{\{a\},\{b\}\}\subseteq f(\M,\mmod(E_1))$ would imply
$\emptyset\in f(\M,\mmod(E_1))$).
If $f(\M,\mmod(E_1))=\{\{a\}\}$ or $\{\{a,b\}\}$, then let $E_2=\{K_1\}$. 
Then,  $\mod(\Delta_{\mu}(E_2))=\{\{a\}, \{a,b\}\}= \mod(\revdelta_{\mu}(E_2))$, 
and  $\mod(\revdelta_{\mu}(E_1) \wedge \revdelta_{\mu}(E_2)) = \{\{a\}\}$ or $\{\{a,b\}\}$.
Furthermore, $\mod(\Delta_{\mu}(E_1\sqcup E_2))= \{\{a\},\{a,b\}\}= \mod(\revdelta_{\mu}(E_1\sqcup E_2))$, thus violating $(\ic 6)$.
If $f(\M,\mmod(E_1))=\{\{a, b\},\{a\}\}$, then let $E_2=\{K_2\}$. Then,  $\mod(\Delta_{\mu}(E_2))=\{\{b\}, \{a,b\}\}= \mod(\revdelta_{\mu}(E_2))$, 
and  $\mod(\revdelta_{\mu}(E_1) \wedge \revdelta_{\mu}(E_2)) = \{\{a,b\}\}$. 
Furthermore, $\mod(\Delta_{\mu}(E_1\sqcup E_2))= \{\{b\},\{a,b\}\}=\mod(\revdelta_{\mu}(E_1\sqcup E_2))$,  and  thus $(\ic 6)$ does not hold.
\end{proof}

\section{Conclusion}
We have investigated to which extent known merging operators can be refined to work within propositional fragments. Compared to revision, this task is more involved since merging operators have many  parameters that have to be  taken into account, and the field of investigation is very broad.

We have first defined desired properties any refined merging operator should satisfy and  provided a  characterization of all refined merging operators. 
We have shown that the refined merging operators preserve the basic merging postulates, namely $(\ic 0)$--$(\ic 3)$. The situation is more complex for the other postulates. For the postulate $(\ic 4)$ we have provided a sufficient condition for its preservation  by a refinement (fairness). We have shown that this condition is not necessary and it would be interesting to study how to weaken it in order to get a necessary and sufficient condition. For the other postulates, we have focused on two representative families of distance-based merging operators that satisfy the postulates $(\ic 0)$--$(\ic 8)$. For these two families the preservation of the postulates $(\ic 5)$ and $(\ic 7)$  depends on the used refinement and it would be interesting to obtain a necessary and sufficient condition for this. In contrast, there is no hope for such a condition for $(\ic 6)$ and $(\ic 8)$, since we have shown that any refinement of merging operators belonging to these families violates these postulates. 

As future work we are interested in solving the open question of whether Proposition~\ref{prop:IC567_violated_for_fair_inHorn} can be extended to the Krom fragment or whether there exists a fair refinement for Krom which satisfies $(\ic 5)$ or $(\ic 7)$.
We also plan a thorough investigation of the complexity of  refined merging operators.

\section{Acknowledgments}
This work has been supported 
by PHC Amadeus  project No 29144UC (OeAD FR 12/2013),
by the Austrian Science Fund (FWF): P25521, and by the Agence Nationale de la Recherche,
ASPIQ project
 ANR-12-BS02-0003.


\end{document}